[1994/12/01]
\documentclass[leqno,fleqn]{article}
\usepackage[margin=0.75in]{geometry}	

\usepackage{apacite}

\usepackage{color}
\usepackage[largesmallcaps]{kpfonts}
\usepackage{microtype}
\usepackage{amsmath,amsthm}
\usepackage{xy}				
\usepackage{graphicx}			

\chardef\bslash=`\\ 

\newtheorem{thm}{Theorem}[section]
\newtheorem{cor}[thm]{Corollary}
\newtheorem{lem}[thm]{Lemma}

\theoremstyle{definition}

\theoremstyle{remark}

\newcommand{\tend}{t_{\mathrm{end}}}
\newcommand{\Tend}{T_{\mathrm{end}}}

\usepackage{dsfont}

\newcommand{\y}[1]{y_{#1}}
\newcommand{\x}[1]{x_{#1}}
\newcommand{\f}{\varphi}
\renewcommand{\r}{\rho}
\newcommand{\yrf}{y_{\f \r}}

\newcommand{\grad}{\nabla}

\newcommand{\RR}{\mathbb{R}}	

\newcommand{\mb}[1]{#1}
\newcommand{\mr}[1]{\mathrm{#1}}
\newcommand{\mt}[1]{\mathtt{#1}}
\newcommand{\mc}[1]{\mathcal{#1}}

\usepackage{etoolbox}
\usepackage{marginnote}				
\usepackage{datetime}
\reversemarginpar
\newbool{lblprint}
\boolfalse{lblprint}
\newcommand{\lbl}[1]{\label{#1} \ifbool{lblprint}{\marginnote{\tiny\textsf{{#1}}}[0cm]}{}} 

\newcommand{\eval}[2][\right]{\relax
  \ifx#1\right\relax \left.\fi#2#1\rvert}

\title{Discrete symbolic optimization and Boltzmann sampling by continuous neural dynamics: Gradient Symbolic Computation}
\author{\textsc{Paul Tupper} \\
\textsc{Paul Smolensky} \\
\textsc{Pyeong Whan Cho}}
\date{\small{\today \hspace{1cm} \currenttime \hspace{1cm} \textsf{\jobname}}}

\begin{document}

\maketitle

\markboth{ Tupper, Smolensky, Cho}{Dynamics of Gradient Symbolic Computation}
\renewcommand{\sectionmark}[1]{}

\abstract{
 Gradient Symbolic Computation is proposed  as a means of solving discrete global optimization problems using a neurally plausible continuous stochastic dynamical system. Gradient symbolic dynamics involves two free parameters that must be adjusted as a function of time to obtain the global maximizer at the end of the computation. We provide a summary of what is known about the GSC dynamics for special cases of settings of the parameters, and also establish that there is a schedule for the two parameters for which convergence to the correct answer occurs with high probability. These results put the empirical results already obtained for GSC on a sound theoretical footing.}

\section{Introduction: Unifying symbolic and neural optimization}

\subsection{The historical and contemporary context of the work}
 
 The recent spectacular successes and real-world deployment of neural networks in Artificial Intelligence (AI) systems have placed a premium on understanding the knowledge in such networks and on explaining their behavior \cite{voosen2017ai}. Such explanations are difficult in part because of the very different formal universes in which neural networks and human understanding live. Usually, the state space of a neural network is taken to be a continuous real vector space. But an explanation of how any system performs a task must, by definition, make contact with the human conceptual system \cite{keil1989concepts}. Higher cognition deploys (at least approximately) discrete categories and rules --- as characterized formally by traditional theories, rooted in symbolic computation, within cognitive science \cite{chomsky1965aspects, newell1972human, marcus2001algebraic} and AI \cite{nilsson1980principles}. Unifying neural and symbolic computation promises a path not only to explainable neural networks but also to a new generation of AI systems and cognitive theories that combine the strengths of these two very different formalizations of computation \cite{Eliasmith2012large}. 

The work presented here contributes to the mathematical foundations of one approach to such neural-symbolic unification \cite{smolensky2006harmonic}. In this approach, a single processing system can be formally characterized at two levels of description. At a fine-grained, lower level of description, it takes the form of a specially-structured type of neural network; at a coarser-grained, higher level of description, it constitutes a novel type of symbolic computation in which symbols within structures have continuously variable levels of activity. This architecture defines \emph{Gradient Symbolic Computation} (GSC) \cite{smolensky2014, cho2017incremental}. 

A key aspect of the GSC approach characterizes processing as \emph{optimization}: given an input, processing constructs an output which maximizes a well-formedness measure called \emph{Harmony} ($H$) \cite{smolensky1983schema, smolensky1986information}. Thus GSC networks instantiate an `energy-based approach' \cite{hopfield1982neural, cohen1983absolute, ackley1985learning, lecun2007energy} (with Harmony corresponding to negative energy). 

Crucially, in GSC a discrete subset of the real vector space of network states is isomorphic to a combinatorial space of discrete symbol structures: an explicit mapping `embeds' each symbol structure as a distributed numerical vector \cite{ smolensky1990tensor, legendre1991distributed}, as sometimes done in classic and contemporary Deep Neural Network modeling \cite{pollack1990recursive, plate1993holographic, socher2010learning}. Harmony can be computed at both the neural and the symbolic levels \cite{ smolensky2006harmony}: the symbolic-level Harmony of a symbol structure equals the neural-level Harmony of the vector that embeds that structure.

At the symbolic level, an important case is when the Harmony of a symbol structure measures the wellformedness of that structure --- the extent to which it satisfies the requirements of a \emph{grammar}: for a given input, the symbol structure with maximal Harmony (typically unique) is the grammar's output. 
This is \emph{Harmonic Grammar}, in its deterministic form \cite{legendre1990harmonic}. 
In Probabilistic Harmonic Grammar \cite{culbertson2013cognitive}, the maximum-Harmony output is the most likely one, but other symbol structures also have non-zero probabilities of being the output: the probability that the output equals a given symbol structure $\mt{S}$ is an exponentially increasing function of its Harmony $H(\mt{S})$.  

Note that although we will sometimes refer to Harmonic Grammar, the results in the paper are general, and no assumptions regarding grammar per se are made. 

\subsection{The technical problem}

To summarize the preceding discussion:
\begin{eqnarray} \label{eq:desired-output}
\textrm{The desired output}
\end{eqnarray}
\begin{enumerate}
\item \textit{Deterministic problem: global optimization.} Output the symbol structure that (globally) maximizes Harmony (typically unique).
\item \textit{Probabilistic problem: sampling.} Output a symbol structure $\mt{S}$ with probability proportional to $e^{H(\mt{S})/T}$.
\end{enumerate}
The randomness parameter $T > 0$, the \emph{computational temperature}, determines how concentrated the probability distribution is on the maximum-Harmony structure: the lower $T$, the greater the concentration.
In the limit $T \rightarrow 0$, the probabilistic problem reduces to the deterministic problem.
Indeed, we will see that, when processing   a single input, we need to decrease $T$ to zero during the computation,
 leading the network to converge to a stable state.
For solving the sampling problem, $T$ decreases only to the level desired for the Boltzmann distribution being sampled.
(In the computational linguistics literature, Probabilistic Harmonic Grammars are known under the name Maximum-Entropy or Maxent grammars \cite{goldwater2003learning, hayes2008maximum}, and $T$ typically has a single, non-dynamic value.)

The particular problem of interest is:
\begin{eqnarray} \label{eq:desired-computation}
\textrm{The desired computation}
\end{eqnarray}
\begin{enumerate}
\item \textit{Optimization Problem.} Through a continuous dynamical process in the embedding space (a continuous processing algorithm for the underlying neural network), converge to the vector embedding the symbol structure that maximizes Harmony (presumed unique).
\item \textit{Sampling Problem.} Through a continuous dynamical process in the embedding space, converge to the vector embedding a symbol structure $\mt{S}$ with probability proportional to $e^{H(\mt{S})/T}$, i.e., produce a sample from the \emph{Boltzmann distribution} defined by the Harmony function $H$.
\end{enumerate}
%
%
GSC is a framework for models of how language-users might meet the requirements of grammar use in a neurally plausible way, searching continuously through a real vector space in which the discrete candidate outputs are embedded, rather than jumping directly from one discrete alternative to the next.


\cite{smolensky2014} proposed a GSC neural dynamics and conjectured that these dynamics solve the problems \eqref{eq:desired-computation}.
The correctness of the conjecture is required to validate the fundamental mode of explanation deployed in GSC, which uses the Harmony of output candidates to reason about their relative probabilities as outputs of the GSC network dynamics.
In this paper we establish formal results concerning the correctness of the method, although the dynamics we study differs technically (but not conceptually) from that of \cite{smolensky2014}.
(The new dynamics studied here has been used in cognitive models in \cite{cho2016bifurcation, cho2017incremental}.)

With the results presented here in place, in the grammatical setting the processing behavior of the underlying neural network can justifiably be formally understood in terms of maximizing symbolic grammatical Harmony, with knowledge of the grammar being realized in the interconnection weights that determine the Harmony of neural states \cite{smolensky1993harmonic}. (We note that the problem of \emph{learning} such weights is not addressed here.)
 
 Thus the work presented below bears on several issues of considerable interest, some quite general, others more specialized:
 \begin{itemize}
\item opening a path towards explainable neural network processing
\item	developing the mathematical foundations of an architecture for unifying neural and symbolic computation
\item	providing theorems concerning global optimization and sampling over discrete symbol structures through continuous neural computation
\item	validating the grounding in neural computation of a grammatical formalism widely employed in phonological theory.
 \end{itemize}
 
(Regarding the last point, the grammatical theory directly relevant here, Harmonic Grammar, is itself currently attracting increasing attention of linguists \cite{pater2009weighted, potts2010harmonic}. Importantly, Harmonic Grammar also provides the neural grounding of \emph{Optimality Theory} \cite{prince1993optimality}, which grew out of Harmonic Grammar and can be viewed as a crucial special case of it \cite{prince1997optimality}. Optimality Theory has been deployed by linguists to analyze all levels of linguistic structure, from phonetics to pragmatics ($\mathtt{http://roa.rutgers.edu}$); it is a dominant paradigm in the theory of phonology \cite{prince1993optimality, mccarthy1993prosodic, mccarthy2008optimality}, which characterizes the complex discrete symbol structures that mentally represent the physical realization of linguistic expressions, motoric and perceptual \cite{chomsky1968sound, goldsmith1990autosegmental}. GSC departs from previous Harmonic Grammar and Optimality Theory work in exploiting \emph{gradient} symbolic representations.)


The innovation implicit in \cite{smolensky2014} and explicit in \cite{cho2016bifurcation} was to introduce another type of Harmony in addition to the Harmony defining grammatical wellformedness which we would like to optimize or sample from. The new type of Harmony is called \emph{quantization Harmony} $Q$.
$Q$ assigns 0 Harmony to all neural states which are the embedding of a discrete symbolic structure, so adding it to grammatical Harmony does not change which discrete structure has maximal Harmony.
But $Q$ penalizes with negative Harmony neural states that are not near symbolically-interpretable states. 
The \emph{total Harmony} $\mc{H}$ is a weighted sum of the grammatical Harmony and quantization Harmony; maximizing this requires finding states that have high grammatical Harmony \emph{and} have symbolic interpretation (avoiding a penalty from $Q$).
The weight assigned to $Q$ in the total Harmony is called $q$; it measures the strength, relative to grammatical Harmony, of the requirement for symbolic interpretability (or `discreteness').
The quantity $q$ turns out to play a central conceptual and formal role in the theory: in order to produce a discretely-interpretable output, $q$ increases during the processing of an input. 
The interplay between the dynamics of change of $q$ and the dynamics of neural activation is at the heart of GSC.

Thus the stochastic dynamical equations we study here have two dynamic parameters: the degree of randomness, $T$, and the level of discreteness, $q$. 
By choosing a schedule for how $T$ and $q$ are changed over time during computation of an output, the system can be shown to perform global optimization --- that is, enter into the global maximum of grammatical Harmony in a finite period of time with high probability.
More generally, under other schedules for $T$ and $q$, the system can be shown to perform Boltzmann sampling --- that is, to terminate near discrete outputs with a probability that is exponential in the Harmony of those outputs.

The plan for the paper goes as follows.
Section \ref{sec:set-up} formally specifies the computational task that needs to be solved in GSC, describing it in terms of maximizing a Harmony function $H$ over the states in a discrete `grid', or more generally, producing discrete outputs in accordance with the Boltzmann distribution at some non-zero temperature $T$.
 We introduce the Harmony function $\mathcal{H}_q$ defined on all points in the vector space of neural states and specify a system of stochastic differential equations whose trajectories seek optima of $\mathcal{H}_q$. 
 In Section \ref{sec:basic} we state some of the basic properties of the local maxima of the function $\mathcal{H}_q$ as $q \rightarrow \infty$. 
 Section~\ref{sec:mathresults} establishes several basic mathematical results about the behavior of our stochastic differential equation for various limiting cases of $q$ and $T$, establishing formal senses in which the GSC dynamics solves problems \eqref{eq:desired-computation}. 
 We use these results in Section~\ref{sec:cooling_schedule} to derive the existence of cooling schedules that with arbitrarily high probability lead the system to be arbitrarily close to the Harmony maximum. 
 In Section~\ref{sec:discussion} we conclude with a discussion of the use of our framework for Gradient Symbolic Computation.

\section{The formal problem and preview of results} \label{sec:set-up}

The formal embedding of discrete symbol structures in GSC employs \emph{tensor product representations} \cite{smolensky1990tensor}.
This method starts by choosing a \emph{filler/role decomposition} of the target set $\mc{S}$ of discrete structures:
each particular structure $\mt{S} \in \mc{S}$ is characterized as a set of \emph{filler/role bindings} $\mc{B}(\mt{S})$,
each binding identifying which symbol $\mt{f}$ from an alphabet $\mc{F}$ fills each structural role $\mt{r} \in \mc{R}$, where the roles $\mc{R}$ determine the structural type of instances of $\mc{S}$.
For example, if $\mc{S}$ is the set of strings over alphabet $\mc{F}$, a natural filler/role decomposition employs roles $\mc{R} = \{ \mt{r}_{k} \}$ identifying the $k^{\mr{th}}$ element of the string.
Then, e.g., $\mc{B}(\mt{abc}) = \{ \mt{b}/\mt{r}_{2}, \mt{a}/\mt{r}_{1}, \mt{c}/\mt{r}_{3} \}$; the filler/role binding $\mt{b}/\mt{r}_{2}$ encodes that the second element of $\mt{abc}$ is $\mt{b}$.
Let $F = |\mc{F}|$ and $R = |\mc{R}|$ respectively denote the number of distinct fillers (symbols) and roles.

Once a filler/role decomposition has been chosen for $\mc{S}$, the remaining steps in defining a tensor product representation are to choose a vector-embedding mapping $\psi_{\mc{F}}$ for the set of fillers, $\psi_{\mc{F}}: \mc{F} \rightarrow V_{\mc{F}}$ and another such mapping for the set of roles, $\psi_{\mc{R}}: \mc{R} \rightarrow V_{\mc{R}}$.
The tensor product representation is thus determined: it is the vector-embedding mapping
\begin{equation}
\psi_{\mc{S}}: \mc{S} \rightarrow V_{\mc{S}} \equiv V_{\mc{F}} \otimes V_{\mc{R}},
\hspace{.25in} \psi_{\mc{S}}: \mt{S} \mapsto 
\sum_{\mt{f}/\mt{r} \in \mc{B}(\mt{S})} \psi_{\mc{F}}(\mt{f}) \otimes \psi_{\mc{R}}(\mt{r})
\end{equation}

Below, the set of discrete structures $\mc{S}$ of interest will be the set of candidate outputs for a Harmonic Grammar: e.g., a set of strings, or trees, or attribute-value structures.
We will assume that the structural roles have been defined such that no discrete structure may have more than one symbol in any given role (i.e., position).
It is convenient for the ensuing analysis to further assume that in any instance $\mt{S} \in \mc{S}$, every role $\mt{r} \in \mc{R}$ has \emph{exactly} one filler.
This assumption does not entail any real loss of generality as we can always assume there is a `null filler' \ \O\  that fills any role that would otherwise be empty (e.g., the role $\mt{r}_{3}$ for the length-2 string $\mt{ab}$).
Furthermore we assume that $\mc{S}$ permits any symbol $\mt{f}$ in $\mc{F}$ to fill any role $\mt{r} \in \mc{R}$.
This too will not entail any loss of generality in the following analysis as we can always, if necessary, add to the Harmony function terms that assign high penalties to structures in which certain disfavored symbols fill particular roles; then, despite being in the set of candidate outputs $\mc{S}$, candidates with those disfavored filler/role bindings will be selected with vanishing probability as outputs of the Harmony-maximizing dynamics we will define.

Under these assumptions of convenience we can concisely characterize the set of discrete output candidates $\mc{S}$ to be exactly those in which each role is filled by exactly one symbol. 
Then under a tensor product representation embedding, the vectors encoding elements of $\mc{S}$ constitute the following set (where $k \mapsto \mt{r}_{k}$ is an enumeration of the elements of $\mc{R}$ and $j \mapsto \mt{f}_{j}$ an enumeration of $\mc{F}$): 
\begin{equation}
\mc{G} \equiv  
\left\{ \sum_{\mt{f}/\mt{r} \in \mc{B}(\mt{S})} \psi_{\mc{F}}(\mt{f}) \otimes \psi_{\mc{R}}(\mt{r})  \ | \  \mt{S} \in \mc{S}  \right\}
\equiv
\left\{ \sum_{\mt{f}/\mt{r} \in \mc{B}(\mt{S})} f \otimes r \ | \  \mt{S} \in \mc{S}  \right\}
=
\left\{ \sum_{k \in 1:R} f_{\varphi(k)} \otimes r_{k} \ | \  \varphi: (1:R) \rightarrow (1:F)  \right\}
\end{equation}
where here and below we abbreviate the \emph{filler vector} $\psi_{\mc{F}}(\mt{f}) $ as $f$ and the \emph{role vector} $\psi_{\mc{R}}(\mt{r}) $ as $r$; $1:n$ abbreviates $1, 2, ..., n$.

We refer to the set $\mc{G}$ as \emph{the grid}. 
In the simplest non-trivial case, where $\mc{F} = \{ \mt{f}_{1}, \mt{f}_{2} \}$ and $\mc{R} = \{ \mt{r}_{1}, \mt{r}_{2} \}$,  $\mc{G}$ is the set of four vectors
\begin{equation*}
\mb{f_{1}} \otimes \mb{r_{1}} + \mb{f_{1}} \otimes \mb{r_{2}},   \ \ \ \ \   \mb{f_{1}} \otimes \mb{r_{1}} + \mb{f_{2}} \otimes \mb{r_{2}},  \ \ \ \ \  \mb{f_{2}} \otimes \mb{r_{1}} + \mb{f_{1}} \otimes \mb{r_{2}}, \ \ \ \ \ 
\mb{f_{2}} \otimes \mb{r_{1}} + \mb{f_{2}} \otimes \mb{r_{2}}.
\end{equation*}
These vectors can be rewritten as
$\{
([\mb{f_{+}} + \alpha_{1} \mb{f_{-}}] \otimes \mb{r_{1}} ,
 [\mb{f_{+}} + \alpha_{2} \mb{f_{-}}] \otimes \mb{r_{2}} )
\}$
where
$(\alpha_{1}, \alpha_{2}) \in \{ (-1, -1), (-1, +1), (+1, -1), (+1, +1) \}$ 
and
$\mb{f_{+}} \equiv (\mb{f_{1}} + \mb{f_{2}})/2, 
\mb{f_{-}} \equiv (\mb{f_{1}} - \mb{f_{2}})/2$. Thus they lie on the vertices of a parallelogram in $\mathbb{R}^N$.

Finally, assume that the filler vectors $\{ \psi_{\mc{F}}(\mt{f}) | \mt{f} \in \mc{F} \} = \{ \mb{f}_{\f} | \f \in 1:F \}$, and the role vectors
$\{ \psi_{\mc{R}}(\mt{r}) | \mt{r} \in \mc{R} \} = \{ \mb{r}_{\r}| \r \in 1:R \}$,
are respectively bases of $V_{\mc{F}}$ and $V_{\mc{R}}$. This implies that each of these sets is linearly independent, which is essential, and that they span their respective vector spaces, which is just a convenient assumption.  The independence of $\{ f_\f\}_{\f \in 1:F}$ and $\{ r_\r\}_{\r \in 1:R}$ implies the independence of the set $\{ f_\f \otimes r_\r\}_{\f \in 1:F, \r \in 1:R}$.
Then it follows that $\{ f_\f \otimes r_\r\}_{\f \in 1:F, \r \in 1:R}$ is a basis for $V_{\mc{S}} = V_{\mc{F}} \otimes V_{\mc{R}} \cong \RR^{F} \otimes \RR^{R} \cong \RR^N$, where $N$ is the product of the number of fillers ($F$) and the number of roles ($R$). 
Henceforth, variables such as $\f, \f'$ etc.\ will be assumed to range over $1:F$, while $\r, \r'$ etc.\ range over $1:R$.
Thus the general vector $y \in \RR^N$ can be characterized by the coefficients $\{\yrf\}$:
\begin{equation}\lbl{yTP}
\mathbf{y} = \sum\nolimits_{\f \r} \yrf f_\f \otimes r_\r
\end{equation}

Correspondingly, a vector $x$ on the grid $\mathcal{G}$ has the form:
\begin{equation}\lbl{xTP}
x = \sum\nolimits_{\f \r} \x{\f \r} f_\f \otimes r_\r 
= \sum\nolimits_{\r}  f_{\f_x (\mt{r}_{\r})} \otimes r_\r \in \mathcal{G}
\end{equation}
Since $x \in \mathcal{G}$, for each role $\mt{r}_{\r}$ there is a single filler $\mt{f}_{\f(\mt{r}_{\r})}$ which has coefficient $x_{\f({\mt{r}_{\r}}) \r} = 1$; all other $\x{\f \r} = 0$. 
Algebraically, we can express this statement as: 
\begin{equation}\lbl{xIdents}
x \in \mathcal{G} \textrm{ if and only if for all } \f, \r:
\end{equation}
\vspace{-15pt}
\[
\begin{array}{lrll}
\textrm{a.} & \x{\f \r}(1 - \x{\f \r}) & = & 0
\\ 
\textrm{b.} & \sum\nolimits_{\f'} \x{\f' \r}^2 - 1 & = & 0
\end{array}
\]


Focussing now on the Optimization Problem in GSC (\ref{eq:desired-computation}.1), the problem is to find the point $x \in \mathcal{G}$ that maximizes $H(x)$, i.e.,
 \begin{equation} \label{eqn:fundOpt}
\max_{x \in \mathcal{G}} H(x)
\end{equation}
 where $H$ is the plain Grammatical Harmony (without any quantization Harmony component).  
 
 A fundamental assumption of GSC is that $H$ takes the form of a quadratic function
\begin{equation} \lbl{Hdef}
H(y) = \frac{1}{2} y^T W y + b^T y.
\end{equation}
This is the quadratic function that we want to maximize over $\mc{G}$, a finite set of discrete points. 


%

Since we want to use neural computation to model how the brain solves the problem of maximizing Harmony $H$ on the grid $\mathcal{G}$, we choose to implement our computations in units with continuous activation varying continuously in time. We need to show how our discrete optimization problem can be encoded in this continuous way.
GSC employs a standard way of doing this: we create a function $Q(y)$ that penalizes distance away from the grid. 
In particular, we choose $Q(y)$ to be zero on the grid and negative off the grid. Then, for large values of $q$, we maximize $\mathcal{H}_q= H + q Q$ over all of $\mathbb{R}^N$. Then we can either send $q$ to infinity or find some other way to round finally-computed states to the nearest grid point to obtain a maximizer on the grid.

With these goals and (\ref{xIdents}) in mind, we define Quantization Harmony $Q$  as:
\begin{equation}\lbl{Qdef}
Q (y) \equiv - \frac{1}{2} \sum\nolimits_\r {\left\{ 
{\left[ \sum\nolimits_\f \yrf^2 - 1 \right]^2} +\sum\nolimits_\f {\yrf^2 (1 - \yrf)^2} \right\}}
\end{equation}
The first term in braces is $0$ if and only if for the particular role $\r$, the sum of the squared $y_{\f \r}$     is $1$, which is condition (\ref{xIdents}b) . The second term is $0$ if and only if $y_{\r \f}=0$ or $1$, which is condition (\ref{xIdents}a). Together, this ensures 
that $Q(y) \leq 0$ for all $y \in \mathbb{R}^N$ and $Q(y)=0$ if and only if $y \in \mathcal{G}$. 

Together, $H$ and $Q$ define the Total Harmony $\mathcal{H}_q$, parameterized by $q$:
\begin{equation}\lbl{TotalH}
\mathcal{H}_q \equiv  H + q Q.
\end{equation}

Now a basic strategy for solving \eqref{eqn:fundOpt} suggests itself. Choose a `large' value of $q$. Solve the continuous optimization problem 
\begin{equation} \label{eqn:contOpt}
\max_{y \in \mathbb{R}^N} \mathcal{H}_q( y)
\end{equation}
For any finite $q>0$, the maximizer of $\mathcal{H}_q$ will not be a point on $\mathcal{G}$. 
But, as we will show, it will  be close to the optimal solution on the grid. 

In optimization this approach is sometimes known as a penalty method \cite[Ch.\ 17]{nocedal2006}. The problem is first solved with a small value of $q$ and then $q$ is gradually increased while the solution is updated. The procedure stops when increasing $q$ further does not change the solution significantly.

A penalty method does not however solve a fundamental aspect of our problem which is that we want global maxima of $\mathcal{H}_q$ rather than local maxima. For large $q$, every point in $\mc{G}$ is close to a separate local maximum of $\mc{H}_q$, and so we can not use a simple steepest ascent method to find global maxima.


One solution to the problem of finding global maxima is to use \emph{simulated annealing} \cite{kirkpatrick1983}. Simulated annealing is a popular method  when the function to be optimized has many local optima but a global optimum is desired. Standard algorithms for finding local optima typically involve going `uphill' until a local maximum is found. Simulated annealing combines uphill moves with occasional downhill moves to explore more of the state space.  During simulated annealing the parameter $T$, known as `temperature', is decreased with time. When temperature is high, the algorithm is almost equally likely to take downhill steps as uphill steps. As $T$ is decreased, the algorithm becomes more and more conservative, eventually only going uphill. A wealth of computational experiments and theoretical analysis has shown simulated annealing to be effective for many global optimization problems. 

Thus our approach is to combine both a penalty method (in that we choose a sufficiently large value of $q$) and simulated annealing to find solutions to \eqref{eqn:fundOpt}.  See \cite{robini2013} for a similar framework.


To implement our method for solving problem \eqref{eqn:fundOpt}, we introduce the following system of stochastic differential equations:
\begin{equation} \lbl{GSCsde}
d y =  \grad \mathcal{H}_q (y) dt + \sqrt{2 T} dB
\end{equation}
 where $B$ is  a standard  $N$-dimensional  Brownian motion. 
 This equation describes how the activations, given in the vector $y$, change in time. 
 The first term $\grad \mathcal{H}_q (y)$ indicates there is a net drift of $y$ in the direction of increasing $\mathcal{H}_q$. The second term indicates that on top of this drift there is noise being continually added to the values of $y$. 
 Together they show that the system is undergoing a noisy random walk biased towards going uphill with respect to $\mathcal{H}_q$. When $T$ is large, the noise is large compared to the uphill motion, whereas when $T$ is small the randomness is negligible. 
 
 
 If we rewrite these equations in the form
\begin{equation} \lbl{GSCsdeweird}
d y =  - \grad (-\mathcal{H}_q (y))  dt + \sqrt{2 T} dB
\end{equation}
we see that it is a standard equation of mathematical physics known as either (overdamped) Langevin diffusion \cite{roberts1996,mattingly2002}, Brownian dynamics \cite{schuss2013}, or a gradient system with additive noise \cite{givon2004}.

For a system of the form \eqref{GSCsde}, we say $\pi$ is an invariant measure if when $X(s)$ is distributed according to $\pi$ then $X(t)$ is distributed according to $\pi$ for any $t>s$. In other words, once the the state of the system is distributed according to $\pi$, it remains distributed according to $\pi$. Invariant measures are extremely important for systems that are \emph{ergodic}, that is, where the system has a unique invariant measure and, given an initial distribution (including a deterministic one), the distribution of the system converges to the invariant measure.
Among other results, in Section \ref{sec:mathresults} we will see that the GSC dynamics are ergodic for any finite fixed $T$ and $q$.

Under reasonable conditions on $\mathcal{H}_q$, the dynamics has an invariant distribution
$\exp( \mathcal{H}_q (y) / T)$.  
The particular assumptions GSC makes about the function $\mathcal{H}_q$ guarantee the following results, as we will show in Section~\ref{sec:mathresults}.

{\bf Fact 1.} For any fixed $q>0, T>0$, the density $\exp( \mathcal{H}_q (y) / T)$ is the unique invariant distribution of \eqref{GSCsde} and can be normalized to be a probability density. For all initial conditions $y(0)$, the probability distribution of $y(t)$ converges to this unique invariant probability measure exponentially fast. (Note that the rate of convergence will depend on $T$ and $q$.) 

{\bf Fact 2.} For fixed $T>0$, as $q \rightarrow \infty$ all the probability mass in the equilibrium distribution will be concentrated near points in the grid $\mathcal{G}$. The probability of being near point $x \in \mathcal{G}$ is proportional to $\exp( \mathcal{H}_q (x) / T)$ --- as required for solving the Sampling Problem (\ref{eq:desired-computation}.2).


{\bf Fact 3.} For fixed $q$, there is a cooling schedule for $T$ such that with probability 1, $y(t)$ will converge to the maximum of $\mathcal{H}_q$---as required for solving the Optimization Problem (\ref{eq:desired-computation}.1)

In Section~\ref{sec:cooling_schedule} we'll use these results to establish a schedule for $q$ and $T$ that will suffice for the process to converge to the global maximum on the grid.  Since this schedule will take infinite time to converge, we will also describe finite-time schedules such that for any 
$\epsilon>0$ there is a combined schedule for $q$ and $T$ such that with probability at least $1-\epsilon$, $y(t)$ will converge to the maximum of $H$ on the grid $\mathcal{G}$.

In what follows we will make use of two different notions of convergence of random variables.
Suppose $\pi$ is a probability measure for a random variable $X \in \mathbb{R}^N$, so that 
\[
\mathbb{P}( X \in A) = \pi(A)
\]
for all  $A \subseteq \mathbb{R}^N$ belonging to the collection $\mc{M}$ of measurable sets. 
And suppose $\nu$ is the measure for another random variable $Y$.
We define the \emph{total variation} metric \cite{gibbs2002}  between $X$ and $Y$ (or equivalently between $\pi$ and $\nu$) to be
\[
\| \pi - \nu \| = \sup_{A \in \mathcal{M}} \left| \pi(A) - \nu(A) \right|.
\]
Given a sequence of random variables $X_n$ with probability measures $\pi_n$, $n \geq 1$, we say $\pi_n$ converges to $\pi$ in total variation if $\|\pi_n - \pi\| \rightarrow 0$ as $n \rightarrow \infty$.
Another, weaker, definition of convergence is that of \emph{weak convergence} where we say $\pi_n$ weakly converges to $\pi$ if for all bounded continuous $f \colon \mathbb{R}^N \rightarrow \mathbb{R}$
\[
\int f d \pi_n \rightarrow \int f d \pi.
\]


%
%

\section{Basic Properties of the Harmony Function as $q \rightarrow \infty$} \label{sec:basic}

Recall that the harmony function $\mathcal{H}_q$ is given by
\[
\mathcal{H}_q(y) = H(y)  + q Q(y)
\]
where $H$ is the quadratic function in \eqref{Hdef} and $Q$ is given in \eqref{Qdef}. $Q$ has the property that it is $0$ at grid points and negative elsewhere, so it has global maxima at all points on the grid. For large $q$, we expect $\mathcal{H}_q$ to act like $qQ$, and so we might hope that $\mathcal{H}_q(y)$ has local maxima near the grid points. We would like the global maximum of $\mathcal{H}_q$ to be near the grid point where $H(y)$ is greatest. 

Our hopes are well founded, as the following result shows.  First, in order to use the implicit function theorem, we compute the first two derivatives of $\mathcal{H}_q(y)$.
%
\begin{eqnarray} \lbl{Weqns}
\grad H & = & {Wy + b} 
\\ \nonumber
\grad^2 H & = & {W}
\\ \nonumber
\grad \mathcal{H}_q(y) & = & {Wy + b} + q \grad Q(y)
\\ \nonumber
\grad^2 \mathcal{H}_q(y) & = & {W} + q \grad^2 Q(y) 
\end{eqnarray}
where
\begin{eqnarray} \lbl{gradQ}
\left[ \grad Q(y) \right]_{\f\r}  
= -2 \yrf \left[ \sum\nolimits_{\f'} \y{\f' \r}^2 - 1 \right] 
- \yrf (1 - \yrf) (1 - 2\yrf),
\end{eqnarray}
\begin{eqnarray} \lbl{D2Q}
\left[ \grad^2 {Q}(y) \right] _{\f'\r',\f\r} 
& = & - \delta_{\r\r'}  \left\{ 2 \delta_{\f\f'} \left[ {\sum\nolimits_{\f''} {{{(y_{\f"\r})}^2} - 1} } \right] + 4  y_{\f\r} y_{\f'\r}  +\delta_{\f\f'} \left[ {1 - 6(y_{\f\r})(1 - y_{\f\r})} \right] \right\}.
\end{eqnarray}
An important observation is that for grid points $x$ we have 
\[
\left[ \grad^2 Q(x) \right]_{\f'\r',\f\r} = -  \delta_{\r\r'} \delta_{\f\f'} (1 + 4 x_{\f\r}^2),
\]
because on the grid, ${\sum\nolimits_{\f''} {{{(x_{\f"\r})}^2} - 1} }  = 0$,  $x_{\f\r} (1-x_{\f\r})=0$, and for any $\r$, $x_{\f\r} x_{\f'\r}=0$ unless $\f=\f'$.
So the Laplacian of $Q$ is diagonal and negative definite at grid points.


The following theorem shows that the local maxima of $\mc{H}_q$ are within $\mathcal{O}(q^{-1})$ of the nearest point in $\mc{G}$. As well, the values of $\mc{H}_q$ are within $\mathcal{O}(q^{-1})$ of the values of $H$ on $\mc{G}$.

\begin{thm} \label{thm:usingIpmFun}
Let $x^* \in \mathcal{G}$. There is a neighborhood $\mathcal{N}$ of $x^*$ and a $\bar{q}$  such that for  $q \geq   \bar{q}$ there is a local maximum $x_q$ of $\mathcal{H}_q$ with
\[
x_q = x^* - q^{-1} \left[ \grad^2 Q(x^*) \right]^{-1} \grad H( x^*)+ \mathcal{O}(q^{-2}),
\]
and
\[
\mathcal{H}_q(x_q) = H(x^*) - q^{-1} \grad H(x^*)^T \left[ \grad^2 Q(x^*) \right]^{-1} \grad H( x^*) + \mathcal{O}(q^{-2})
\]
\end{thm}
\begin{proof}
Note that since $\mathcal{H}_q$ is smooth, local optima satisfy $\grad \mathcal{H}_q(x)=0$. 

To study how solutions to this equation depend on $q$ we let $\epsilon = q^{-1}$ and study the equivalent equation
\[
h(x,\epsilon) \equiv \epsilon \grad H(x) + \grad Q(x) = 0.
\]
We use the implicit function theorem: see \cite[p. 631]{nocedal2006}. Since $h(x,\epsilon)=0$ has solution $h(x^*,0)=0$, $h$ is twice continuously differentiable everywhere, and $\grad_x h(x,\epsilon)$ is nonsingular at the point $(x,\epsilon)=(x^*,0)$, we have that we can uniquely solve for $x$ in a neighborhood of $x^*$ in terms of $\epsilon$ for all $\epsilon$ sufficiently close to $0$. Furthermore, $\epsilon \mapsto x_\epsilon$ is twice continuously differentiable and 
\begin{eqnarray*}
\frac{d x_\epsilon}{d \epsilon}(0) = - [ \grad_x h(x^*,0) ]^{-1} \grad_\epsilon h(x^*,0)  = - [\grad^2 Q(x^*)]^{-1} \grad H(x^*).
\end{eqnarray*}
Putting it back in terms of $q^{-1}$ and building the Taylor expansion of $x_q$ gives the first result. The second result comes from substituting the first result into the Taylor expansion for $\mathcal{H}_q$ about $x^*$.
\end{proof}

As a straightforward corollary of the previous theorem, we have that the global maximum of $\mathcal{H}_q$
is close to the global maximum of $H$ on $\mc{G}$ for large $q$.

\begin{cor} \label{cor:largeq}
Suppose $H$ has a unique global maximum on $\mathcal{G}$ and the gap between the global maximum $x^*$ of $H$ on $\mathcal{G}$ and the second highest local maximum is at least $g >0$. 
Let $\eta_1, \eta_2>0$ be given. Then there is a $\bar{q}$ such that for $q\geq \bar{q}$, $\mathcal{H}_q$ has a unique global maximum $x_q$ satisfying both
\[
\| x_q  - x^* \| \leq \eta_1, 
\]
and $\mathcal{H}_q(x_q)$ is at least $g-\eta_2$ away from the value of the second highest maximum. 
\end{cor}
\begin{proof}
Follows straightforwardly from the previous theorem.
\end{proof}

\section{Mathematical Results for the GSC Dynamics}\label{sec:mathresults}

In this section we establish the mathematical results about GSC dynamics that we outlined in Section~\ref{sec:set-up}.

\subsection{For fixed $q, T$, convergence to invariant distribution as $t \rightarrow \infty$}

We use the framework of \cite{roberts1996} to obtain the following result:

\begin{thm} \label{thm:robertsuse}
For the stochastic differential equation defined by \eqref{GSCsde} with fixed $q$ and $T$
\begin{enumerate}
\item $y(t)$ is defined for all time
\item there is a unique invariant probability measure $\pi(y)= C \exp[ \mathcal{H}_q(y)/T ]$ where $C$ is a normalizing constant depending on $q$ and $T$.
\item for any fixed initial condition $y(0)$ the distribution of $y(t)$ converges exponentially to $\pi$ as $t \rightarrow \infty$
\end{enumerate}
\end{thm}

Here exponential convergence of the distribution of $y(t)$ to $\pi$ means there are constants $R_{y_0} < \infty$ and $\rho<1$ such that
\begin{equation} \lbl{expconv}
\| P^t(y_{0},\cdot) - \pi \| \leq R_{y_0} \rho^t
\end{equation}
for $t \geq 0$. $P^t(y_0,A)$ is the probability that $y(t)$ is in $A$ given $y(0)=y_0$. The constant $R_{y_0}$ in general depends on $y_0$.

\begin{proof}

%
%
%
%
%

The paper  \cite{roberts1996} studies equations of the form
\begin{equation} \lbl{eqn:robertsform}
dx_\tau = \frac{1}{2} \grad  \log \pi (x_\tau) d\tau + dB_\tau
\end{equation}
where $x \in \mathbb{R}^N$, $\pi \colon \mathbb{R}^N \rightarrow \mathbb{R}$, and as before $B_\tau$ is $N$ dimensional standard Brownian motion. The authors assume that $\pi$  is everywhere non-zero, differentiable and integrates to 1. 
 This equation has invariant density $\pi$ as can be checked via the Fokker-Planck equation \cite[Ch. 5]{gardiner}. 

We can transform our equation \eqref{GSCsde} into the form of \cite{roberts1996} using change of variables, $\tau = 2T t$, and $x=y$. This gives
\[
dx = \frac{1}{ 2 T} \grad \mathcal{H}_q (x) d\tau + dB .
\]
Now this can be transformed into \eqref{eqn:robertsform} by letting
\[
\frac{1}{2} \log \pi(x) =\frac{1}{ 2 T} \mathcal{H}_q(x)
\]
or 
\[
\pi(x) = \exp [ T^{-1} \mathcal{H}_q(x) ].
\]
So everything \cite{roberts1996} proved for \eqref{eqn:robertsform} with $\pi = \exp(\mathcal{H}_q /T)$ applies for our equation, though with a $T$-dependent rescaling of time. 

Theorem 2.1 of \cite{roberts1996} asserts that if $\grad \log \pi$ is continuously differentiable, and if for some $M, a, b < \infty$,
\[
\left[ \grad \log \pi(x) \right] \cdot x  \leq a \|x\|^2 + b, \ \ \ \ \mbox{for all } \|x\| > M,
\]
then the dynamics are almost surely defined for all time, and the probability density function of the process converges to $\pi$ in the total variation norm for all initial conditions. 

In our case $\grad \log \pi = \frac{1}{T} \grad \mathcal{H}_q$. So we need an expression for $[\grad \mathcal{H}_q(y)] \cdot y$. Since $\mathcal{H}_q(y) = H(y) + q  Q(y)$, we first compute,
\[
\grad H(y) \cdot y =  (Wy + b) \cdot y = y^T W y  + b^T y   \leq w_{\max} \|y\|^2 + \|b\| \, \|y\| \leq (w_{\max}+ \|b\|) \, \|y\|^2
\]
if $\|y\| > 1$, where $w_{\max}$ is the maximum eigenvalue of $W$.
Then, using the expression for $\grad Q$ from \eqref{gradQ}, we have
\[
[  q \grad Q(y) ] \cdot y =- q \sum_{\f\r} \yrf \left\{ 2 \yrf \left[ \sum_{\f'} y_{\f' \r}^2 - 1 \right] + \yrf (1-\yrf)(1- 2 \yrf) \right\}
\]
The first term on the right can be bounded as
\begin{eqnarray*}
- q \sum_{\f \r}  \yrf \left\{ 2  \yrf \left[ \sum_{\f'} y_{\f' \r}^2 - 1 \right] \right\} & \leq  & - 2 q \sum_{\f\r}  \yrf^2 [ \yrf^2 -1] \\
& \leq & 2 q \sum_{\f\r} \yrf^{2}  = 2 q   \|y\|^2.
\end{eqnarray*}
For the second term on the right we have
\begin{eqnarray*}
-q \sum_{\f\r} \yrf \left\{ \yrf ( 1- \yrf) (1- 2 \yrf) \right\} \leq \frac{q}{8} \sum_{\f\r} \yrf^2 = \frac{q}{8} \|y\|^2
\end{eqnarray*}
where we have used that $-(1-z)(1-2z) \leq \frac{1}{8}$ for all $z$. Putting these bounds together gives, if $\|y\| >1$,
\[
\frac{1}{T} \grad \mathcal{H}_q (y) \cdot y  \leq \frac{1}{T} \left[ (w_{\max} + \|b\|) \|y\|^2  + \frac{17}{8} q \|y\|^2 \right].
\]
Theorem 2.1 of \cite{roberts1996} then gives  results 1 and 2 in the statement of Theorem~\ref{thm:robertsuse}


To demonstrate exponential convergence to $\pi$ we  use Theorem 2.3 of \cite{roberts1996}.
They state the exponential convergence is guaranteed for \eqref{GSCsde} if 
\begin{enumerate}
\item there exist an $S >0$ such that $|\pi(x)|$ is bounded for $|x|\geq S$.
\item there exists a $d$, $0< d< 1$, such that 
\end{enumerate}
\begin{equation} \label{eqn:neededForExponential}
\liminf_{|x| \rightarrow \infty} (1-d)  \| \grad \log \pi(x)  \|^2 + \grad^2 \log \pi(x) >0.
\end{equation}
The first condition is true even for $S=0$, since $\pi(x)$ is bounded. For the second condition,
recall that $\grad \log \pi = \frac{1}{T} \grad \mathcal{H}_q$. So 
\[
\grad^2 \log \pi = T^{-1} \mbox{Trace}[ \grad^2  \mathcal{H}_q(y)] = T^{-1} \mbox{Trace}[ W + q \grad^{2} Q(y)]. 
\]
If we take the trace of \eqref{D2Q} we find that $\grad^2 \log \pi$ is negative for large $y$ but that no term grows faster than quadratically in $\yrf$. On the other hand,
\begin{eqnarray*}
\grad \log \pi(y) & = & \frac{1}{T} \grad \mathcal{H}_q(y) =\frac{1}{T} [W y + b  +  q \grad Q(y)], \\
\left[ \grad \log \pi(y) \right]_{\f \r} & = & -2 \frac{q}{T}  \yrf \left[ \sum_{\f'} y_{\f' \r}^2 + \yrf^2 \right] + \mathcal{O}(\|y\|^2).
\end{eqnarray*}
So 
\begin{equation} \label{niceInequality}
\| \grad \log \pi(y) \|^2 \geq 4 \frac{q^2}{T^2} \sum_{\f \r} \yrf^6 + \mathcal{O}(\|y\|^5).
\end{equation}
So the left-hand side in \eqref{eqn:neededForExponential} grows like a 6th order polynomial in $\|y\|$ for $\|y\| \rightarrow \infty$ and $d \in [0,1)$ and so the condition is satisfied. Theorem 2.3 of \cite{roberts1996} then gives result 3 in the statement of our theorem.
\end{proof}

Exponential convergence sounds good but the $\rho$ in \eqref{expconv} may be quite close to $1$ for large $q$ and small $T$.  To give a rough estimate of how the rate of convergence scales with $q$ and $T$ we perform in informal analysis using the Arrhenius formula (see \cite[p. 141]{gardiner} or \cite[p. 334]{vankampen}). The Arrhenius formula gives an order of magnitude estimate for how long it takes a diffusion to exit one optimum and enter another. 
Generally, for a process to reach the equilibrium distribution from a fixed initial condition, the process has to visit representative points in the state space more than once. So the time to get from one local maximum to another gives a very conservative lower bound on how long it will take the process to reach equilibrium.

The one-dimensional version of the Arrhenius formula gives that the expected time to exit a state $a$ by getting over a saddle at point $b$ is
\[
\tau = \frac{2\pi}{\sqrt{ U''(a) |U''(b)|}} \exp \left[\frac{U(b) - U(a)}{T}\right]
\]
where $U$ is the potential. Adapting the result to our multidimensional case, we know that $U=\mathcal{H}_q$ scales like $q$ as $q \rightarrow \infty$. 
So the expected time to converge to equilibrium at least goes like $\exp[k q/T]/q$ for some constant $k$. This time diverges very rapidly as either $q \rightarrow \infty$ or $T \rightarrow 0$.

\subsection{For fixed $T$, as   $q \rightarrow \infty$, convergence to Boltzmann distribution}

Here we consider fixed $T$ and see what happens to the equilibrium distribution as $q \rightarrow \infty$.
From the previous subsection we know that the equilibrium distribution for fixed $q, T$ is 
\[
Z_q^{-1} \exp[ \mathcal{H}_q(x)  /T] = Z_q^{-1} \exp[ ( H(x) + q Q(x) )/T]
\]
where $Z_q$ is a $q$- and $T$-dependent constant chosen to yield a probability distribution. (We suppress the dependence on $T$ in this section since we will not vary $T$.)

In what follows let $B(x,\eta)$ be all the points within distance $\eta$ of $x$. Let $\{ x_i \}$ be all the points where $Q(x_i)=0$, that is, the grid points $\mathcal{G}$.

\begin{thm}\lbl{prob-grid}
There exists a  $\eta_0>0$ such that  for all $\eta \in (0,  \eta_0]$
\[
\lim_{q\rightarrow \infty} \int_{B(x_i,\eta)} Z_q^{-1} \exp[ ( H(x) + q Q(x) )/T] dx = Z^{-1}  \exp[H(x_i)/T]
\]
where $Z$ is a normalizing constant depending on $T$ but not on $q$ or $i$.
Furthermore, 
\[
\lim_{q \rightarrow \infty} \int_{\mathbb{R}^N \setminus \cup_i B(x_i,\eta )} Z_q^{-1}  \exp[ ( H(x) + q Q(x) )/T] dx = 0.
\]
\end{thm}

Informally: for large $q$, all the probability mass of the equilibrium distribution is concentrated about the $x_i$. In this limit, each $x_i$ has probability mass proportional to $\exp[H(x_i)/T]$. So this result effectively shows that GSC is Probabilistic Harmonic Grammar in the $q\rightarrow \infty$ limit, solving the Sampling Problem (\ref{eq:desired-computation}.2).

\begin{proof}
This result is obtained straightforwardly from \cite[Prop B2]{kolokoltsov2007}. We quote the result there in full:
Let 
\[
I(h) = \int_\Omega f(x) \exp[ -S(x)/h ] dx
\]
where $\Omega$ is any closed subset of the Euclidean space $\mathbb{R}^d$, the functions $f$ and $S$ are continuous and $h \in (0,h_0]$ for some positive $h_0$. 

We need the following assumptions:
\begin{enumerate}
\item the above integral is absolutely convergent for $h=h_0$.
\item $S(x)$ is thrice continuously differentiable
\item $\Omega$ contains a neighborhood of the origin. As well, $S(x) >0$ for $x \neq 0$ and $S(0)=0$.
\item $\grad S(0)=0$ and $\grad^2 S(0)$ is positive definite.
\item $\liminf_{x \rightarrow \infty,  x \in \Omega} S(x) >0$
\item there exists positive $r$ such that
\begin{enumerate}
\item $\inf\{ S(x) \colon \Omega \setminus B(0,r) \}= \min\{ S(x) \colon x \in \partial B(0,r) \}$
\item$\grad^2 S(x) \geq \Lambda$ for all $x \in B(0,r)$  and some positive real $\Lambda$
\item $U(h_0) \subset B(0,r)$, where $U(h) = \{ x \colon x^T \grad^2 S(0) x \leq h^{2/3} \}$.
\end{enumerate}
\end{enumerate}

Let $D = \det \grad^2 S(0)$. 

Proposition B2 states:
Let the above assumptions hold and let $f$ be two times continuously differentiable and let $S$ be four times continuously differentiable. Then 
\[
I(h) = (2\pi h)^{d/2} \left( f(0) D^{-1/2} + h [D^{-1/2} \delta_1(h) 
+ \Lambda^{-d/2} \delta_2(h) ] \right) + \delta_3(h),
\]
where 
\begin{itemize}
\item $| \delta_1(h)|$ has an $h$-independent bound,
\item $| \delta_2(h)|$ has an $h$-independent bound,
\item $|\delta_3(h)|$ converges to zero as $h$ goes to zero faster than any polynomial. 
\end{itemize}
There are explicit expressions for all these bounds, but we do not need them here.

To apply this result to our integrals (without the normalizing constant), we need to change some variables.
So we let
\begin{eqnarray*}
f(x) & = &  \exp \left[ H(x) /T \right] \\
S(x) & = & -Q(x)/T \\
h & = & 1/q
\end{eqnarray*}
Also, for each integral we imagine translating the functions so that $x_i$ is at the origin.
Our $\Omega$ corresponds to $B(x_i,\eta_0)$, where $\eta_0$ is a constant we define shortly.

So do our $f(x)$ and $S(x)$ satisfy the conditions of the theorem?
We go through the  conditions stated above in the same order.
\begin{enumerate}
\item Immediately true because $B(x_i,\eta)$ is bounded and  the integrand is bounded on bounded sets.
\item Follows because $Q(x)$ is thrice continuously differentiable.
\item This is true as long as we choose $\eta_0$ to be small enough so that $B(x_i,\eta_0)$ includes only one local maximum of $Q(x)$.
\item Recall that the $x_i$ are local maxima of $Q$, so $\grad Q(x_i)=0$.  As shown in Section~\ref{sec:set-up}, $\grad^2 Q$ is negative definite and so $\grad^2 S(x_i)$ is positive definite. 
\item This is vacuously true since $\Omega=B(x_i,\eta_0)$ is bounded.
\item This is another constraint on how big $\eta_0$ is. Let $\eta_0$ be small enough so that $\grad^2 S(x) \geq \Lambda$ for some $\Lambda>0$. This is possible since $\grad^2 S(x)$ is positive definite. This is condition 6(b). Condition 6(a) follows since $S(x)$ is then convex on $\Omega$. (The only $\grad S(x)=0$ point is $x_i$. So the minimum of $S(x)$ on $\Omega \setminus B(0,r)$ is either on the inner boundary or the outer boundary. If it is on the outer boundary then there must be an even lower point on the inner boundary by convexity.)
To see that condition 6(c) holds for some $h_0$ and $r$, note that $U(h_0) \subseteq B(0,h_0^{1/3} \Lambda^{1/2})$. So just let $h_0$ be small enough so that $r:= h_0^{1/3} \Lambda^{1/2} < \eta_0$.
\end{enumerate}


Furthermore, just $H(x)$ and $Q(x)$ being smooth satisfies the additional conditions.
So we get that
\begin{equation} \label{eqn:insideballs}
\int_{B(x_i,\eta)}  \exp \left[ ( H(x) +q Q(x))/T \right] dx 
=
(2 \pi/ q)^{d/2} \left(  \exp \left[  H(x_i) / T \right] D^{-1/2} + q^{-1} \delta(q) \right)
\end{equation}
where $\delta(q)$ has a $q$-independent bound, for all sufficiently large $q$.
This is true for all $i$.
Note that in our case
\[
D =  T^{-1} \det -\grad^2 Q(x_i) = T^{-1} \prod_{rf}  (1 + 4 x_{rf})^2 = T^{-1} 5^R,
\]
where $R$ is the number of roles,
which takes the same values for all grid points.

What about the rest of $\mathbb{R}^d$? Another useful result in \cite{kolokoltsov2007} is the following.
If $\inf_{\Omega} S(x) \geq M$ then 
\[
I(h) \leq C \exp [ -M/h ]
\]
where $C$ is some $h$-independent constant.

For our case this becomes
\begin{equation} \label{eqn:outsideBalls}
\int_{\mathbb{R}^d \setminus \cup_i B(x_i,\eta)}  \exp \left[ ( H + q Q)/T \right] dx 
\leq C \exp[ -M q].
\end{equation}

By the definition of the normalization constant $Z_q$ we know that
\begin{eqnarray*}
1 & = & Z_q^{-1} \int_{\mathbb{R}^d} \exp [ (H(x)+q Q(x))/T ] \, dx  \\
& = & Z^{-1}_q  \sum_i \int_{B(x_i,\eta)}  \exp [ (H(x)+q Q(x))/T ] \, dx + Z^{-1}_q  
\int_{\mathbb{R}^d \setminus \cup_i B(x_i,\eta)}  \exp \left[ ( H + q Q)/T \right] dx.
\end{eqnarray*}
Taking the limit as $q \rightarrow \infty$ and observing the $\exp[-Mq ]$ decreases faster than any power of $q$, we obtain
\[
1 = \lim_{q \rightarrow \infty}  \left[  Z^{-1}_q (2 \pi q)^{d/2}  D^{-1/2} \right] \sum_i \exp[ H(x_i)/T].
\]
If we let
\[
Z= \lim_{q \rightarrow \infty} Z_q (2 \pi q)^{-d/2} D^{-1/2}
\]
then we obtain the desired result.

\end{proof}

%
%

\subsection{For fixed $q$, as  $T \rightarrow 0$, convergence to the global maximum of $\mathcal{H}_q$}

Reducing $T$ to $0$ over time---simulated annealing---is a common technique for finding global optima of functions \cite{kirkpatrick1983,vanlaarhoven,hajek}.
For a sufficiently slow cooling schedule, it is known that the process will converge to the global maximum of $\mathcal{H}_q$. There are many versions of this result; here we use \cite{chiang1987} as it is closest to our framework, using a diffusion process to optimize a continuous function on $\mathbb{R}^N$. We will use their results to show that $y(t)$ converges to the maximum of $\mc{H}_q$ as $t \rightarrow \infty$, for sufficiently slow cooling,
showing that the GSC dynamics can solve the Optimization Problem (\ref{eq:desired-computation}.1)

Here we state \cite{chiang1987}'s main result. 
They consider the diffusion equation
\[
dX(t) = -\grad U(X(t)) dt  + \sigma(t) dW(t), \ \ \ \ X(0)=x_0.
\]
where $W(t)$ is $N$-dimensional Brownian motion.
Let $\pi^\epsilon$ be the probability distribution proportional to $\exp[ - 2 U(y) / \epsilon^2]$ and let $\pi^\epsilon$ have a unique weak limit as $\epsilon \rightarrow 0$, $\pi^\epsilon \rightarrow \pi$.
Let $U$ be a twice continuously differentiable function from $\mathbb{R}^n$ to $[0,\infty)$ such that
\begin{enumerate}
\item $\min_y U(y)= 0$
\item $U(y) \rightarrow \infty$ and $|\grad U(y)| \rightarrow \infty$ as $|y| \rightarrow \infty$.
\item $\lim_{|y| \rightarrow \infty} |\grad U(y) |^2 - \grad^{2} U(y) > -\infty$.
\end{enumerate}
Finally, assume that $\sigma(t) <1$ and $\sigma^2(t) = c/\log t$ for large $t$. Then there is a $c_0$ such that for $c>c_0$, for any bounded continuous function $f$
\[
p(0,x_0,t,f) \rightarrow \pi(f), \  \ \ \mbox{as } t \rightarrow \infty.
\]
Here $p(0,x_0,t,f)$ is the expected value of $f(X(t))$ given $X(0)=x_0$.



The $U$ in \cite{chiang1987} corresponds to our $- \mathcal{H}_q$ except that $\min_y -\mathcal{H}_q(y) \neq 0$. This can be obtained by replacing $\mathcal{H}_q(y)$ with $\mathcal{H}_q (y) - \max_{y'} \mathcal{H}_q(y')$. This does not change the dynamics or the location of the maxima at all.
Continuing to apply their framework, we need to take $\epsilon^2/2 = T$. We let $\pi_T$ be the distribution  $\exp[-\mathcal{H}_q/T]$ normalized to be a probability distribution (i.e.\ to have total mass 1). We need to show that $\pi_T$ has a weak limit as $T \rightarrow 0$. This again can be shown using the results of \cite{kolokoltsov2007}. We let $Z_T$ be the normalizing constant such that $\pi_{T} = Z_T^{-1} \exp[-\mathcal{H}_q/T]$. We let $\mathcal{G}_q$ be the local maxima of $\mathcal{H}_q$ and
let $\mathcal{G}'_q \subseteq \mathcal{G}_q$ be the points where it attains its global maxima. (Typically, we expect there to be only one point in $\mathcal{G}'_q$.
  The result is that the limit distribution as $T \rightarrow 0$ is equal point masses distributed at all the points of $\mathcal{G}'_q$. Of primary interest to us is the situation when there is a unique global maximum and hence the limiting distribution is a single point mass at the global optimum.)

\begin{lem}
For sufficiently large $q$,
the probability distribution with density $Z_T^{-1} \exp[-\mathcal{H}_q(y)/T]$  converges weakly to 
\[
\frac{1}{|\mathcal{G}'_q|} \sum_{x \in \mathcal{G}'_q} \delta_x(y). 
\]
\end{lem}
\begin{proof}
We choose $q$ large enough and $\eta$ small enough so that
for all $x \in \mathcal{G}'_q$ and all $y \in B(x,\eta)$, $\grad^2 \mathcal{H}_q(y) \geq \Lambda > 0$ for some $\Lambda>0$.

We again use Kolokoltsov's result (given in the proof of Theorem \ref{prob-grid}). Let $y \in \mathcal{G}'_q$ and $\eta>0$. We let $\Omega= B(y,\eta)$.
We let $f(x)=1$ and $S(x)=\mathcal{H}_q(x)$. We go through the same six assumptions as before.

We perform a translation so that $x$ is at the origin, and $\mathcal{H}_q(x)=0$. We choose $\eta$ small enough here. 

\begin{enumerate}
\item Immediately true because $B(x,\eta)$ is bounded and the integrand is bounded on bounded sets.
\item Follows because $\mathcal{H}_q$ is smooth.
\item Follows because $\eta$ is small enough to guarantee only one local minimum.
\item Follows because $x$ is at the origin, $x$ is a global minimum, $\grad^2 \mathcal{H}_q(x) \geq \Lambda>0$. 
\item Follows since $\Omega$ is bounded.
\item Follows from the same logic as our earlier application of Kolokoltsov's result.
\end{enumerate}

So we obtain
\[
\int_{B(y,\eta)} \exp[ - \mathcal{H}_q(x)/T] \, dx = (2 \pi T)^{d/2} (D^{-1/2} + \mathcal{O}(T))
\]
where $D = \grad^2 \mathcal{H}_q(y)$. Since $D$ does not depend on $y$, each point in $\mathcal{G}_q'$ gets the same amount of mass.

Kolokoltsov's second result shows that the probability mass outside the neighborhoods of $y \in \mathcal{G}_q'$ goes to zero, giving us the result.
\end{proof}

Now we can state our main result for this section, showing how GSC solves the Optimization Problem (\ref{eq:desired-computation}.1):
\begin{thm} \lbl{thm:simanal}
Consider a sufficiently large $q$ for which $\mathcal{H}_q$ has a unique global optimum $y$.
There is a $c_0$ such that for $c> c_0$ if $T(t) =  c/\log(t)$ then the solution of \eqref{GSCsde}  converges to $y$ with probability 1.
\end{thm}
\begin{proof}
We need to show that the three conditions in \cite{chiang1987}'s theorem hold in our case where $U(y)=-\mathcal{H}_q(y)+ \max_{y' \in \mathbb{R}^N} \mathcal{H}_q(y')$. 
The first condition requires $\min_x U(x)=0$, which follows from the definition of $U$.
The second condition requires $U(x)$ and $|\grad U(x)|$ go to $\infty$ as $|x| \rightarrow \infty$. Observe that in $-\mathcal{H}_q$ the first term of $-q Q$ grows like a fourth power of $|y|$ and the second term is positive. Likewise, the harmony term $H(y)$ grows at most quadratically, so $-\mathcal{H}_q(y)$ goes to infinity as $|y| \rightarrow \infty$. That $|-\grad \mathcal{H}_q(y)|$ likewise goes to $\infty$ as can be seen from inequality \eqref{niceInequality}, since $\log \pi (y) = \frac{1}{T} \mathcal{H}_q(y)$. The third condition follows from \eqref{eqn:neededForExponential} being true when $d=0$, as we showed in the proof of Theorem~\ref{thm:robertsuse}.
\end{proof}


An important detail for us is how fast the cooling can occur, and therefore how quickly we can obtain an accurate approximate solution with high accuracy.  This is determined by the constant $c_0$ for which the authors of \cite{chiang1987} give a value which is believed to be optimal to within a factor of 2. We briefly summarize their results here, translating them into the language of maximization.  This will allow us to get an idea of how $c_0$ depends on $q$, an important consideration, since we need $q$ to be reasonably large to ensure that we are close to a grid point.


Recall that $\mathcal{G}_q$ is the set of all stationary points of $\mathcal{H}_q$, i.e. the $x$ where $\grad \mathcal{H}_q(x)=0$.
Let 
\[
I(t,y,x):= \inf_{\psi (0) = y, \psi (t)=x } \frac{1}{2} \int_0^t | \dot{\psi}(s)) - \grad \mathcal{H}_q(\psi(s))|^2 ds,
\]
\[
V(y,x) := \lim_{t \rightarrow \infty} I(t,y,x),
\]
\[
J := \sup_{x,x' \in \mathcal{G}_q} \left( V(x',x) + 2 \mathcal{H}_q(x) \right).
\]
According to \cite{chiang1987}, the optimal $c_0$ (of Theorem~\ref{thm:simanal}) is within a factor of two of $c_* = J$.

$V(y,x)$ can be thought of as a measure of how hard it is to get from $y$ to $x$, taking the most efficient path.
\begin{itemize}
\item If $x$ is directly ``uphill" from $y$, meaning that the trajectory $\dot{\psi}(s)= \grad \mathcal{H}_q( \psi(s) )$, $\psi(0)=y$ eventually reaches $x$, then $V(y,x)=0$.
\item If $x$ is directly ``downhill" from $y$, then an argument in \cite{berglund} shows that $V(y,x)= 2( \mathcal{H}_q(y) - \mathcal{H}_q(x))$. This is ``twice the height of descent".
\item In general, $V(y,x)$ is twice the total amount of downhill descent needed to go from $y$ to $x$.
\end{itemize}

To expand further on the last item, imagine you are a hiker in a landscape who dislikes going downhill. Whatever path you take you keep track of the total number of meters you have to descend as you go along. You don't reduce the number when you go up or at any other time. You call this the descent of a path. Also, given $y$ and $x$ you always take the path that minimizes descent. $V(y,x)$ is twice the descent of the minimum-descent path from $y$ to $x$.

Computing $J$ requires that we find the maximum of $V(x',x)+2\mathcal{H}_q(x)$ over all $x'$ and $x$ in $\mathcal{G}_q$. However, we are mainly interested in finding a lower bound on $J$. So if we find one pair of points $x'$ and $x$ in $\mathcal{G}_q$ such that $V(x',x)+2\mathcal{H}_q(x)$ is large then know we can do no better than that.

In our case, let $x'$ and $x$ be two adjacent maxima of $\mathcal{G}_q$. Suppose one is obtained from the other by changing one filler in one role. The optimal path from $x'$ to $x$ is attained by starting at $x'$, going down to the saddle point between them, and then going up again to $x$. 
When $q$ is large, which will be necessary for the maxima of $\mc{H}_q$ to be close to the grid points in $\mc{G}$, Theorem~\ref{thm:usingIpmFun} states that the values of $\mathcal{H}_q(x')$ and $\mathcal{H}_q(x)$ will be a smaller perturbation of the values of $H$ on the corresponding grid points. This means they will be of order 1, compared to $\mc{H}_q$ on the saddle which will be on the order of $q$. 
Accordingly, the amount of descent between $x'$ and $x$ is roughly linear in $q$.
So the larger $q$ is, the slower the cooling schedule needs to be. 

To turn this into an estimate of a time scale, suppose we want to cool the system from some high value of $T$ to $T=\Tend$. (We will discuss what value of $\Tend$ is necessary to get a required level of accuracy in the next section.)
The previous considerations show that our cooling schedule needs to be like $T(t) = q/\log(t)$ or $t = e^{q/T(t)}$. So the time it takes to run the system to obtain temperature $\Tend$ is on the order of $e^{q/\Tend}$ which grows rapidly as both $q \rightarrow \infty$ and $\Tend \rightarrow 0$. 



%
%
%


\section{Finite-time schedules for $q, T$} \label{sec:cooling_schedule}

Suppose we wish to complete the GSC computation in finite time, which of course the brain must do. A finite time interval necessitates that we do not attain the correct answer with $100$\% accuracy and also that we do not actually reach the optimal grid state $x^{*}$. To formally specify how well we need to do,  we fix parameters $\eta>0$ and $\epsilon>0$ and require that at some $\tend$, $y(\tend) \in B(x^*, \eta)$ with probability at least $1- \epsilon$.

To determine a schedule that attains this, first choose a large enough $q$ so that $x_q^*$, the global maximizer of $\mathcal{H}_q$, is within $\eta/2$ of $x^*$. That such a $q$ exists follows from Corollary~\ref{cor:largeq}. Now starting with $T=T_0$ follow a simulated annealing cooling schedule so that the probability that $y(\tend)$ is within $\eta/2$ of $x_q^*$ is $1-\epsilon$.  From Theorem~\ref{thm:simanal} we know that over an infinite time interval with sufficiently slow cooling we will converge to $x_q^*$ with probability one. So we know that over some finite time interval we will be within $\eta/2$ of $x_q^*$ with  probability at least $1-\epsilon$. 
About how long would this take? 

To determine this,  we first  figure out how big $q$ has to be. It must be large enough so that $x_q^*$, the maximum of $\mc{H}_q$, is close to $x^*$. This requires both that (a) local maxima of $\mc{H}_q$ are close to the grid points, and (b) the values of $\mc{H}_q$ at the local maxima are close enough to the values of $H$ at the grid points. For the first challenge (a), Theorem~\ref{thm:usingIpmFun} shows that we roughly need $q^{-1}$ to be less than some constant times $\eta$. For challenge (b), we need the values of $\mc{H}_q$ at the local maxima to be close to the values of $H$ on the corresponding grid points to within tolerance less than $g/2$ where $g$ is the gap between the global maximum of $H$ on $\mc{G}$ and the next highest point on $\mc{G}$. (Otherwise, the global maximum of $\mc{H}_q$ could shift to being close to another grid point, rather than $x^*$.) So $q^{-1}$ must be smaller than some linear term times $\min(g,\eta)$, or, equivalently, $q \sim \max(g^{-1}, \eta^{-1})$. 


Now, we need to determine how slow simulated annealing needs to go to converge to this $x_q^*$. From the discussion at the end of Section~\ref{sec:mathresults}, we know we have to let temperature $T(t)$ go like $q/ \log (t)$, and it must run on the order of time $e^{q/\Tend}$ where $\Tend$ is the final temperature in the cooling schedule. How low $\Tend$ has to be is determined by how high the probability of being around the correct grid point needs to be.
One way to estimate this time is to see at which temperature $\Tend$ the distribution $\exp ( \mathcal{H}_q /\Tend)$ puts sufficient mass near the point $x^*_q$. To a rough approximation, let's look at the probability mass at $x^*_q$ and then compare it to the probability mass of the next highest points, say $\bar{x}_q$. We know the ratio of their probabilities goes like
\[
\frac{\exp ( \mathcal{H}_q(x_q^*) /\Tend) )}{\exp ( \mathcal{H}_q(\bar{x}_q) /\Tend) )}
=
\exp ( (\mathcal{H}_q(x_q^*)- \mathcal{H}_q(\bar{x}_q)) /\Tend) )
\approx \exp( g/ 2 \Tend).
\]
If we want this probability ratio to be on the order of $\epsilon^{-1}$ (which is necessary to get the probability of being near the correct grid point above $1-\epsilon$), that requires $\Tend \sim g/\log( \epsilon^{-1})$. 

Recall from the previous section that the time to cool to temperature $\Tend$ is like $e^{q/\Tend}$.
Putting these considerations together shows that the time required to obtain a solution within $\eta$ of the true global optimum with probability higher than $1-\epsilon$ is roughly on the order of 
\[
\exp({q/\Tend}) \sim \exp[\max(g^{-1},\eta^{-1})\log(\epsilon^{-1})/g].
\]
Since this expression is somewhat complex, let us make the assumption that $g$ is some fixed parameter significantly larger than $\eta$. Then the time required is on the order of $\epsilon^{-\eta^{-1}}$. As might be expected, the time to convergence to the given tolerances grows as both $\epsilon$ and $\eta$ approach 0. The scaling with $\eta$ is considerably worse than that with $\epsilon$.



\section{Discussion} \label{sec:discussion}

Theorems \ref{thm:simanal} and \ref{prob-grid} show that the GSC dynamics in principle solves the two problems we set out to solve, the Optimization and Sampling Problems of \eqref{eq:desired-computation}, respectively. Practically speaking, however, our estimates suggest that the time required to perform these computations is large. Nonetheless biological processes such as protein folding manage to quickly find good solutions to optimization problems that also are estimated to require lengthy computation. The brain may prove to be another such system. 

One possibility is the brain uses better schedules for $q$ and $T$ than we use here. Our goal was to provide rigorous proof-of-concept for the GSC framework, and we did not yet determine optimal joint schedules for $q$ and $T$.  The schedule proposed in Section~\ref{sec:cooling_schedule} starts with $q$ quite large and leaves it at that value. A natural approach, as used in \cite{smolensky2014} is to begin with $q$ small, thereby biasing the system to be in states close to the global optimum, and then to slowly increase $q$  while decreasing $T$ slowly enough so that the system only explores states close to the global maximum. Alternatively, in $T$ is fixed to a small value and  $q$ alone is updated \cite{cho2017incremental}.  Our analysis in this paper does not cover such schedules for $q$ and $T$, but it is a natural direction for further investigation.

\bibliographystyle{apacite}

\bibliography{dynamicsOfGSC-PS.bib}

\end{document}